\PassOptionsToPackage{prologue,table,xcdraw}{xcolor}
\documentclass[sigconf]{acmart}
\AtBeginDocument{%
  \providecommand\BibTeX{{%
    \normalfont B\kern-0.5em{\scshape i\kern-0.25em b}\kern-0.8em\TeX}}}

\usepackage[ruled, boxed]{algorithm2e}

\usepackage{mathtools}

\usepackage{graphicx}
\usepackage{multirow}
\usepackage{subcaption}
\usepackage[table,xcdraw]{xcolor}
\usepackage{balance}

\copyrightyear{2024}
\acmYear{2024}
\setcopyright{acmlicensed}
\acmConference[KDD '24] {Proceedings of the 30th ACM SIGKDD Conference on Knowledge Discovery and Data Mining }{August 25--29, 2024}{Barcelona, Spain.}
\acmBooktitle{Proceedings of the 30th ACM SIGKDD Conference on Knowledge Discovery and Data Mining (KDD '24), August 25--29, 2024, Barcelona, Spain}
\acmISBN{979-8-4007-0490-1/24/08}
\acmDOI{10.1145/3637528.3672061}
\acmPrice{}

\begin{document}

\title{Conformalized Link Prediction on Graph Neural Networks}

\author{Tianyi Zhao}
\email{tzhao566@usc.edu}
%\email{webmaster@marysville-ohio.com}
\affiliation{%
  \institution{University of Southern California}
  %\streetaddress{P.O. Box 1212}
  \city{Los Angeles}
  %\state{California}
  \country{USA}
}

\author{Jian Kang}
\email{jian.kang@rochester.edu}
\affiliation{%
  \institution{University of Rochester}
  %\streetaddress{1 Th{\o}rv{\"a}ld Circle}
  \city{Rochester}
  \country{USA}}
%\email{larst@affiliation.org}

\author{Lu Cheng}
\email{lucheng@uic.edu}
\affiliation{%
  \institution{Univeristy of Illinois Chicago}
  \city{Chicago}
  \country{USA}
}

\renewcommand{\shortauthors}{Tianyi Zhao, Jian Kang, \& Lu Cheng}

\begin{abstract}
Graph Neural Networks (GNNs) excel in diverse tasks, yet their applications in high-stakes domains are often hampered by unreliable predictions. Although numerous uncertainty quantification methods have been proposed to address this limitation, they often lack \textit{rigorous} uncertainty estimates. This work makes the first attempt to introduce a distribution-free and model-agnostic uncertainty quantification approach to construct a predictive interval with a statistical guarantee for GNN-based link prediction. We term it as \textit{conformalized link prediction.} Our approach builds upon conformal prediction (CP), a framework that promises to construct statistically robust prediction sets or intervals. There are two primary challenges: first, given dependent data like graphs, it is unclear whether the critical assumption in CP --- exchangeability --- still holds when applied to link prediction. Second, even if the exchangeability assumption is valid for conformalized link prediction, we need to ensure high efficiency, i.e., the resulting prediction set or the interval length is small enough to provide useful information. To tackle these challenges, we first theoretically and empirically establish a permutation invariance condition for the application of CP in link prediction tasks, along with an exact test-time coverage. Leveraging the important structural information in graphs, we then identify a novel and crucial connection between a graph's adherence to the power law distribution and the efficiency of CP. This insight leads to the development of a simple yet effective sampling-based method to align the graph structure with a power law distribution prior to the standard CP procedure. Extensive experiments demonstrate that for conformalized link prediction, our approach achieves the desired marginal coverage while significantly improving the efficiency of CP compared to baseline methods. 
Our code is available in \url{https://github.com/Aliciaa-svg/CLP}.
\end{abstract}
% For example, GNN-based link prediction used to help identify potential interactions between different drugs may not be trusted due to its overconfident predictions.






\maketitle

\section{Introduction}
GNNs have emerged as a versatile and powerful model that can operate on graph-structured data, such as social networks~\cite{fan2019graph}, molecular graphs~\cite{jiang2021could}, and knowledge graphs~\cite{liu2021indigo}. Their ability to model complex relationships in graph-structured data has propelled them to the forefront of machine learning research. 
However, one of the major challenges in applying GNNs to real-world problems is the lack of reliable uncertainty estimates for their predictions. A series of recent research has shown mixed results regarding the performance of GNNs~\cite{wang2021confident,hsu2022makes,liu2022confidence}. For example, when used in high-stakes domains such as drug discovery and finance, GNN-based link prediction may not be trusted due to its miscalibrated confidence. 

This work studies uncertainty quantification for GNN-based link prediction. One prominent approach is to construct prediction sets or intervals that provide information about a plausible range of values within which the true outcome is likely to fall. Numerous methods have been put forth to achieve this goal~\cite{wang2021confident,hsu2022makes,zhang2020mix,lakshminarayanan2017simple}. Nevertheless, these methods fall short in terms of offering both theoretical and empirical assurances concerning their validity. Conformal prediction (CP)~\cite{vovk2005algorithmic} has emerged as a promising framework to tackle these limitations and has been applied to various domains, such as natural language processing \cite{giovannotti2022calibration,schuster2022confident,ren2023robots}, causal inference \cite{lei2021conformal}, computer vision~\cite{angelopoulos2021uncertainty,bates2021distribution,belhasin2023principal} and drug discovery~\cite{JMLR:v24:22-1176}. 
It is a framework that promises to construct prediction sets or intervals while ensuring a statistically robust coverage guarantee. That is, given a user-specified miscoverage rate $\alpha\in(0,1)$, CP uses a so-called calibration set of data to produce prediction intervals (often for regression) or prediction sets (often for classification) for the test data, and the resulting set or interval covers the true label or value with probability at least $1-\alpha$. Or, the constructed prediction sets/intervals are theoretically proven to have a guarantee that they will only miss the test outcomes in at most an $\alpha$ fraction of cases.

Further, CP offers the advantage of being compatible with any black-box machine learning model, under the condition that the data follows the principle of statistical exchangeability (e.g., the calibration and test data are exchangeable in conformal prediction). This flexibility alleviates the need for the often violated assumption of independent and identically distributed (i.i.d.) data, particularly common in graph-structured datasets. 
With its simple formulation, weaker assumption, strong theoretical guarantee and
distribution-free nature, a few recent efforts ~\cite{huang2023uncertainty,zargarbashi2023conformal,clarkson2023distribution} have explored to use CP to quantify uncertainty for graph-structured data, with a particular emphasis on tasks like node classification.
Complementary to these prior works, we explore the realm of CP for link prediction, which is related yet inherently different from node classification tasks. 
To illustrate the importance, consider the GNN-based recommender system in a pharmacy store that suggests Over-the-Counter (OTC) medicine to patients. When the system over-confidently recommends inappropriate medicine, patients can be exposed to high risks of adverse effects. In this case, the rigorous prediction interval produced by CP under a predefined error rate (say 10\%) can help assess the reliability of the system.
In this case, a larger interval indicates higher uncertainty, highlighting the need for caution and possibly consulting clinicians for a more informed decision.

In this work, we study a novel problem of \textit{conformalized link prediction (CLP)}. A central challenge arises from the question of whether the critical condition of exchangeability still holds when performing CP at the edge level. In response to this challenge, this work first seeks to thoroughly examine the validity of the exchangeability assumption in GNN-based link prediction. Particularly, we formally define this problem within an inductive setting, where calibration and test edges are excluded from the training process. We then theoretically examine the exchangeability between calibration and test data for link prediction, i.e., whether the distributions of calibration and test data are exchangeable under any permutation. In CP, when the exchangeability assumption is satisfied, the coverage is statistically guaranteed. However, we also need to ensure that the prediction set or the interval length is small enough to be informative. For example, CP can output a trivial interval or set that includes all possible labels, resulting in useless predictions. Existing approaches (e.g., \cite{angelopoulos2020uncertainty,zaffran2022adaptive}) for improving efficiency are inapplicable due to the unique features of graph data.

To address this challenge, we propose to leverage structural properties, one of the most important and unique features in graph data. Graph structures provide vital information and have been shown extremely useful for a variety of graph-related tasks, such as node classification \cite{wu2022nodeformer}, link prediction \cite{yang2021inductive}, and graph classification \cite{lee2018graph}. Therefore, we ask \textit{whether graph structures can provide additional information to improve the efficiency of standard CP for CLP}? One particular type of structural information we investigate is the node degree and its distribution. As a fundamental property in graphs, node degree reflects the connectivity of a node within the graph and provides valuable insights into the structure, function, and behavior of networks \cite{newman2001random}. Informed by this, we conduct a series of empirical analyses and identify an interesting finding: a greater adherence to the power law in the node degree distribution typically leads to significantly increased CP efficiency (\textbf{Figure}. \ref{fig:simulation}). This inspiration drives us to propose a simple approach for harmonizing the degree distribution of a graph with a power-law distribution for more efficient CLP. This is achieved by selectively removing specific edges and utilizing the remaining edges for the CLP process.

In summary, our main contributions are:
\vspace{-2mm}
\begin{itemize}
    \item We propose a novel problem of CLP on GNNs and theoretically establish the condition of exchangeability for CLP, affirming the validity of employing CP for CLP. 
    \item We develop a novel pipeline for efficient CLP via a simple sampling-based approach guided by the fundamental power law distribution of node degrees.
    \item We evaluate the proposed method on real-world graph datasets for the link prediction task. The experimental results suggest that our approach can significantly improve CP's efficiency, especially when the degree distribution in a graph is less adherent to the power law distribution. 
\end{itemize}

\section{Preliminary}
\noindent \textbf{Notation.} Consider a graph $\mathcal{G}=(\mathcal{V,E}_p)$ with $M$ nodes, where $\mathcal{V}=\{1,\cdots,M\}$ and $\mathcal{E}_p=\{e_1,\cdots,e_N \}\subseteq \mathcal{V}\times\mathcal{V}$. 
$X\in\mathbb{R}^{M\times d}$ denotes the node feature matrix.
Let $\mathcal{E}_p$ and $\mathcal{E}_n$ be the positive links set and non-existent links set, respectively. 
The latter is constructed by randomly selecting the same number of non-existent links as the number of positive links from the graph.
$\mathcal{E}=\mathcal{E}_p\cup\mathcal{E}_n=\{e_1,\cdots ,e_{2N}\}$. Each $e\in\mathcal{E}$ is represented as a node pair  $e=(u,v)$ with nodes $u$ and $v$ as its two endpoints.

\vspace{1mm}
\noindent \textbf{The Link Prediction Problem.} Link prediction with GNNs is typically based on representation learning~\cite{acar2009link,kipf2016variational,pan2022neural}. Particularly, we start by obtaining the node representations. Then we derive edge representations based on the learned node representations, often using operations like the dot product between two node representations. These edge representations can then be employed to estimate the likelihood of a link between them. 
GNNs are employed to acquire node representations that encode both the topological structure and the feature information associated with each node.
We measure the performance of a link prediction model by how well it can rank the true links higher than the false ones in the test set. A common metric for this is the ratio of true links that are among the top $K$-ranked links by the model~\cite{hu2020open}.
It is expected that the ranking of scores for positive edges will surpass that of non-existent edges.

\vspace{1mm}
\noindent \textbf{Conformal Prediction.} Conformal prediction (CP) is a distribution-free framework in machine learning and statistical modeling that assigns valid confidence estimates or prediction intervals to the output of predictive models \cite{vovk2005algorithmic}. One of the most common CP methods is split CP~\cite{lei2018distribution}. It acts as a wrapper around a trained base model and uses a set of exchangeable held-out (or calibration) data to construct prediction intervals.

Given an exchangeable set of held-out calibration data
$\{(X_i, Y_i)\}_{i=1}^n$, the goal of CP is to construct a marginal distribution-free prediction interval $C(X_{n+1})\in \mathbb{R}$ that is likely to encompass the unknown response $Y_{n+1}$ with a specified miscoverage rate $\alpha\in[0,1]$: 
\begin{equation}
    \mathbb{P}\{Y_{n+1}\in C(X_{n+1})\} \ge 1-\alpha.
\end{equation}
To achieve this, we first define a non-conformity score function $V:\mathcal{X}\times\mathcal{Y}\rightarrow \mathbb{R}$, which measures the calibration of the prediction for a specific sample, i.e., how true value $y$ conforms to model prediction at $x$.
For each $(X_i,Y_i)$ in the calibration set, we first compute the non-conformity score $V(X_i,Y_i)$.
Next, we define $\widehat{q}$ to be the $\lceil (n+1)(1-\alpha)\rceil / n$-th empirical quantile of $\{V(X_1,Y_1),\cdots,V(X_n,Y_n)\}$. The prediction interval can then be constructed as follows:
\begin{equation}
    C(X_{n+1})=\{y\in\mathcal{Y}: V(X_i,y)\le\widehat{q}\}.
\end{equation}

The Conformalized Quantile Regression (CQR) method~\cite{romano2019conformalized} distinguishes itself as a widely recognized CP technique for creating prediction intervals due to its simplicity and effectiveness.
To apply CQR, we divide the data into a training set $\mathcal{D}_1$ and a calibration set $\mathcal{D}_2$. 
Next, we employ a quantile regression function denoted as $\mathcal{A}$ to fit two conditional quantile functions, namely $\widehat{\mu}_{\alpha/2}$ and $\widehat{\mu}_{1-\alpha/2}$, utilizing the training set.
Subsequently, we compute the non-conformity scores using the calibration set: 
\begin{equation}
    V_i=\max\{\widehat{\mu}_{\alpha/2}(X_i)-Y_i, Y_i-\widehat{\mu}_{1-\alpha/2}(X_i)\},
\end{equation}
for each $(X_i,Y_i)\in \mathcal{D}_2$. 
The scores are then used to calibrate the plug-in prediction interval 
\begin{equation}
    \widehat{C}(x)=[\widehat{\mu}_{\alpha/2}(x),\widehat{\mu}_{1-\alpha/2}(x)].
\end{equation}
More specifically, let $\widehat{q}$ be the $\lceil (|\mathcal{D}_2|+1)(1-\alpha)\rceil / |\mathcal{D}_2|$-th empirical quantile of $\{V(X_1,Y_1),\cdots,V(X_{|\mathcal{D}_2|},Y_{|\mathcal{D}_2|})\}$, the prediction interval for a new input data $X'$ is then constructed as 
\begin{equation}
    C(X')=[\widehat{\mu}_{\alpha_{\alpha/2}}(X')-\widehat{q}, \widehat{\mu}_{\alpha_{1-\alpha/2}}(X')+\widehat{q}].
\end{equation}
% $C(X')=$ $[\widehat{\mu}_{\alpha_{\alpha/2}}(X')-\widehat{q}, \widehat{\mu}_{\alpha_{1-\alpha/2}}(X')+\widehat{q}]$.

\section{Conformalized Link Prediction}
We begin this section by formulating and investigating the validity of conformalized link prediction (CLP), i.e., whether the exchangeability assumption holds for GNN-based link prediction.
It should be noted that this is a critical step for ensuring the statistical guarantee of CP. Yet, there is no prior work that formally studies CP in the context of link prediction, and the adaptation of CP to CLP is nontrivial. Unlike~\cite{huang2023uncertainty}, which explored CP for node classification in a \textit{transductive} setting, our work establishes CP for link prediction within an \textit{inductive} learning framework.
Then we introduce how to leverage Conformalized Quantile Regression (CQR) \cite{romano2019conformalized} for CLP, and further investigate the relationship between the efficiency of CP and the graph's structural property. Based on the empirical analysis, we propose a simple and effective sampling strategy guided by the fundamental power law distribution of node degrees to improve the efficiency of CLP.

\subsection{Exchangeability and Validity of Conformalized Link Prediction}
\label{sec:exchang}
The link prediction problem discussed here naturally fits into an inductive learning framework. To elaborate, we initially divide the set of links, denoted as $\mathcal{E}$, into distinct subsets: the training set ($\mathcal{D}_{train}$), the validation set ($\mathcal{D}_{val}$), the calibration set ($\mathcal{D}_{calib}$), and the test set ($\mathcal{D}_{test}$). Each of these subsets contains an equal number of positive links (indicating existing connections) and negative links (indicating non-existent connections). The GNN model is then trained on a subgraph denoted as $\mathcal{G}'=(\mathcal{V},\mathcal{E}')$, where $\mathcal{E}'=\mathcal{E}_p\cap(\mathcal{D}_{train}\cup\mathcal{D}_{val})$. In other words, the model can access information about all the nodes and their associated features, but it only has access to a portion of the positive links that belong to the training and validation link sets. The objective is to train the model to predict the edges that have not been observed between pairs of nodes. During the training process, we assign label 1 (or 0) to represent positive (or non-existent) edges. The GNN model begins by generating embeddings for edges through message passing and neighborhood aggregation. Subsequently, it produces prediction scores for all edges, which can be used to determine the likelihood of an edge existing between node pairs.

Since the model lacks access to the labels (indicating the status) of links in $\mathcal{D}_{calib}\cup\mathcal{D}_{test}$, any link from this combined set is equally likely to be part of either $\mathcal{D}_{calib}$ or $\mathcal{D}_{test}$. In other words, different choices of calibration sets do not alter the non-conformity scores for any given link. Using GNNs for link prediction on graphs thus adheres to the following permutation invariance condition:
For any permutation $\pi$ of $\mathcal{D}_{calib}\cup\mathcal{D}_{test}$, the non-conformity score $V$ satisfies the following: 
    \begin{equation*}
    \begin{aligned}
        &V(e,y;\{(e_i,y_i)\}_{e_i\in\mathcal{D}_{train}\cup\mathcal{D}_{val}},\{e_i\}_{e_i\in\mathcal{D}_{calib}\cup\mathcal{D}_{test}},\mathcal{G}')\\
        &=V(e,y;\{(e_i,y_i)\}_{e_i\in\mathcal{D}_{train}\cup\mathcal{D}_{val}},\{e_{\pi(i)}\}_{e_{\pi(i)} \in\mathcal{D}_{calib}\cup\mathcal{D}_{test}},\mathcal{G}').
    \end{aligned}
    \end{equation*}  
This states that permuting the order of links within the calibration and test sets does not alter their corresponding non-conformity scores.
%\lu{explain what this long equation means.} 
Therefore, the exchangeability of $\mathcal{D}_{calib}\cup\mathcal{D}_{test}$ is naturally satisfied. To this end, we can present the following proposition, demonstrating that the non-conformity scores exhibit exchangeability with respect to link prediction.

\begin{proposition}\label{pp1}
    In the described inductive setting for link prediction, where the model has access to all node information and features but only a subset of positive links from training and validation sets during training, the unordered set of the scores $[V_i]_{i=1}^{K+L}$ is fixed, where $|\mathcal{D}_{calib}|=K$,  $|\mathcal{D}_{test}|=L$, and $V_i$ denotes the non-conformity score of link $e_i\in\mathcal{D}_{calib}\cup\mathcal{D}_{test}$. 
That is, the non-conformity scores are exchangeable for all $e\in\mathcal{D}_{calib}\cup\mathcal{D}_{test}$.
\end{proposition}

\begin{proof}
Let $f(\cdot)$ be the GNN model trained on the subgraph $\mathcal{G}'$ which produces the node embeddings $H$. Let $g(\cdot)$ denote the function that produces edge embeddings, i.e., $z_{e_i}=g(H_{u_i},H_{v_i})$ for edge $e_i=(u_i,v_i)$.
$h(\cdot)$ is the function that produces the prediction scores based on the edge embeddings, i.e., $s_i=h(z_{e_i})$. Let $v_i=V(s_i, y_i)$ be the non-conformity score for $e_i\in\mathcal{D}_{calib}\cup\mathcal{D}_{test}$.
$f(\cdot)$, $g(\cdot)$, and $h(\cdot)$ are fixed after training.
Thus it is clear that permutating the order of $e\in\mathcal{D}_{calib}\cup\mathcal{D}_{test}$ will not change the resulting non-conformity scores for $e_i\in\mathcal{D}_{calib}\cup\mathcal{D}_{test}$.
The sets of non-conformity scores before and after permutation are exactly the same:
\begin{equation*}
    \{v_1, \cdots, v_{K+L}\} = \{v_{\pi(1)}\cdots v_{\pi(K+L)}\}.
\end{equation*}
\end{proof}

\subsection{CQR for Conformalized Link Prediction}
With the fundamental exchangeability assumption satisfied, we now introduce how CP can be better leveraged to quantify the uncertainty for the link prediction task. Link prediction is framed as a task where a model is trained to produce prediction scores for all missing edges \cite{hu2020open}. The expectation is that the model will rank the prediction scores for positive test edges higher than those for negative edges. In the context of uncertainty quantification for link prediction, it is more appropriate to formulate it as a regression problem and construct a prediction interval instead of viewing it as a binary classification process and creating a prediction set for each unobserved edge.
We therefore propose to leverage Conformalized Quantile Regression (CQR)~\cite{romano2019conformalized} which provides prediction intervals that come with a provable guarantee of coverage probability. For link prediction, we adapt the non-conformity score in CQR to the following:
\begin{equation}
    V(e,y)=\max\{\widehat{\mu}_{\alpha/2}(z_e)-y, y-\widehat{\mu}_{1-\alpha/2}(z_e)\},
\end{equation}
where $\widehat{\mu}_{\gamma}(\cdot)$ denotes the $\gamma$-th conditional quantile function of the edge embeddings $z_e$.
The prediction interval is then constructed as 
\begin{equation}
    C(e)=[\widehat{\mu}_{\alpha/2}(z_e)-\widehat{q}, \widehat{q}-\widehat{\mu}_{1-\alpha/2}(z_e)].
\end{equation}
Based on Proposition~\ref{pp1}, we prove that the validity of coverage $C(e)$ is guaranteed.

\begin{theorem}\label{theorem1}
    Given that $\{v_i\}_{i=1}^{K+L}$ is exchangeable, with error rate $\alpha\in(0,1)$ and $\widehat{q}=\text{Quantile}(v_1,\cdots,v_K;\lceil (K+1)(1-\alpha)/K \rceil)$, the constructed prediction interval for edge $e_{K+j}$ is $C(e_{K+j})=[\widehat{\mu}_{\alpha/2}(z_{K+j})-\widehat{q}, \widehat{q}-\widehat{\mu}_{1-\alpha/2}(z_{K+j})]$, $j=\{1,\cdots,L\}$, satisfying
    \begin{equation*}
        \mathbb{P}\{y_{K+j}\in C(e_{K+j})\}\ge 1-\alpha.
    \end{equation*}
    \label{the:coverage}
\end{theorem}

\begin{proof}
    Let $v_{K+1}$ be the non-conformity score for the test link $(e_{K+1},y_{K+1})$. We have
\begin{equation*}
        \mathbb{P}\{y_{K+1}\in C(e_{K+1})\} = \mathbb{P}\{v_{K+1}\leq \widehat{q}\}.
\end{equation*}
Without loss of generality, we assume that $\{v_i\}_{i=1}^{K}$ is sorted in ascending order, i.e., $v_1\le v_2\le\cdots\le v_K$.
Since $\{v_i\}_{i=1}^{K+1}$ are exchangeable, we have
\begin{equation*}
    \mathbb{P}\{v_{K+1}\le v_{t}\} = \frac{t}{K+1}\quad (1\le t\le K)
\end{equation*}
that is, $v_{K+1}$ is equally likely to fall in anywhere between $v_1,\cdots, v_K$.
Thus the following inequality holds:

\begin{equation*}
\begin{aligned}
    \mathbb{P}\{v_{K+1}\leq \widehat{q}\} &= \mathbb{P}\{v_{K+1}\leq v_{\lceil (K+1)(1-\alpha) \rceil}\} \\
    &= \frac{\lceil (K+1)(1-\alpha) \rceil}{K+1}\\
    &\ge 1-\alpha .
\end{aligned} 
\end{equation*}
\end{proof}
Theorem \ref{the:coverage} suggests that the coverage of the prediction interval in CLP is at least $1-\alpha$ with a rigorous statistical guarantee.

\subsection{Efficiency and Structural Property}
In addition to the coverage rate, another important evaluation metric for CP is efficiency, i.e., the size or the length of the prediction sets or intervals. A smaller size or length suggests a more informative prediction set or interval. To assess the efficiency of CLP, a simple approach is to measure the average length of the prediction interval at a given error rate $\alpha$. A shorter interval length suggests an improved efficiency.

Traditional approaches \cite{angelopoulos2021uncertainty,xu2021conformal} for improving CP efficiency cannot be directly applied as they are for non-graph data (e.g., tabular data or images), leaving the unique characteristics (e.g., structural properties) of graph data largely unexplored. To improve the efficiency of CLP, we explore its potential connection to the graph structural information.  Identifying crucial graph data properties that affect graph learning is an ongoing challenge \cite{yasir2023examining}. Here we specifically focus on node degree distribution, which reflects node connectivity and provides valuable insights into network structure and behavior. It is widely recognized as a significant factor impacting graph model performance \cite{yasir2023examining,li2023metadata}. Other structural properties such as clustering coefficient and connectivity could be valuable to explore in future research.
Next, we reveal a novel and interesting connection between the efficiency of the CLP and the node degree distribution of the underlying graph structure.

We commence our study with a series of experiments on semi-synthetic graphs that exhibit varying degrees of conformity to the power law, as outlined in~\cite{clauset2009power}. These graphs are generated using the method introduced in~\cite{zhao2023unveiling}. Specifically, given a real-world graph, we create $n$ cliques within a given graph by randomly selecting $m$ nodes and connecting all nodes within each clique. By adjusting the values of $m$ and $n$, we can generate synthetic graphs with different levels of conformity to the power law. The Kolmogorov-Smirnov (KS) statistic~\cite{broido2019scale} is employed as a metric to quantify the extent of conformity, where a lower value indicates a higher degree of conformity to the power law. Subsequently, we carry out simulation experiments based on the Amazon Computers dataset~\cite{shchur2018pitfalls}.
% Our primary emphasis lies in examining the degree distribution, a pivotal feature within the graph structure.
% We begin with experiments on synthetic graphs exhibiting varying levels of conformity to the power law~\cite{clauset2009power}.
% These graphs are generated using the approach introduced in~\cite{ding2019interactive}. 
% Specifically, we create $n$ cliques within a given graph by randomly selecting $m$ nodes and then make those nodes fully connected for each clique. 
% By adjusting the values of $m$ and $n$, we can obtain synthetic graphs with different levels of conformity to the power law. 
% The Kolmogorov-Smirnov (KS) statistic~\cite{broido2019scale} serves as the metric to quantify these varying levels of conformity, wherein a lower value signifies higher conformity to the power law.
% We then conduct the simulation experiments based on Amazon Computers~\cite{shchur2018pitfalls} dataset.

\begin{figure}[]
  \centering
  \includegraphics[width=0.7\linewidth]{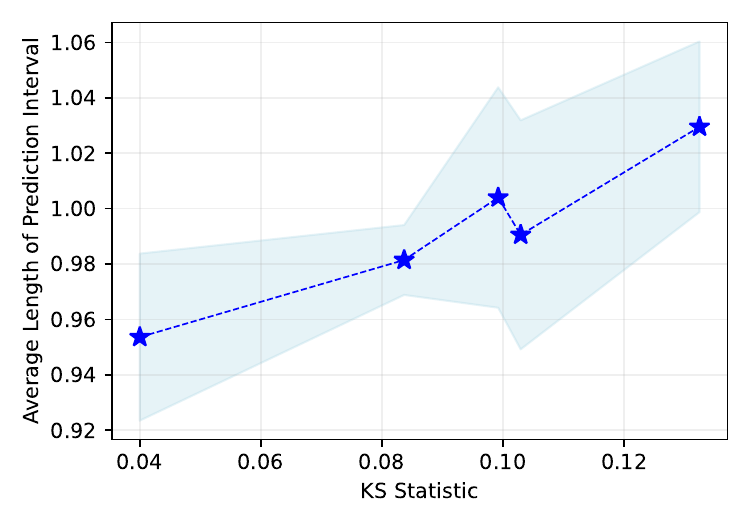}
  %\vspace{-20pt}
  \caption{Simulation study on a semi-synthetic dataset generated from the Amazon Computers dataset~\cite{shchur2018pitfalls}.}
  %\Description{ablation german.}
  \label{fig:simulation}
\end{figure}

Specifically, we select ($m$, $n$) from $\{(25, 20), (50, 20), (75, 20), (100,$ $ 20), (150,20)\}$. For each combination of $(m,n)$, we create five synthetic graphs and apply the CLP procedure described in Sec. 3.2. We then record the average length of the prediction interval as a measure of CP efficiency. To represent the performance of a particular $(m, n)$, we calculate the mean KS statistic value by averaging the KS statistic values of the five synthetic graphs with the same $(m, n)$. 
The simulation results are presented in \textbf{Figure}~\ref{fig:simulation}, in which the horizontal axis value is the averaged KS statistic. The trend is evident: graphs exhibiting higher KS statistic values typically display larger average prediction intervals, suggesting lower efficiency in CLP – that is, less informative prediction intervals. This notable finding inspires us to explore the potential enhancement of CP efficiency when conducting CLP by utilizing edges with a degree sequence that closely aligns with the power law distribution.

%\lu{Add some theoretical analyses if possible.}

\begin{algorithm}
    \caption{Conformalized Link Prediction.}
    \label{algo}
    \LinesNumbered
    \KwIn{\\ \quad Graph $\mathcal{G}=(\mathcal{V},\mathcal{E}_p)$ and links set $\mathcal{E}$.
        \\ \quad Miscoverage level $\alpha\in(0,1).$
        \\ \quad Node embedding algorithm $\mathcal{F}$, edge embedding algorithm $\mathcal{Z}$, edge scoring algorithm $\mathcal{H}$.
        \\ \quad Quantile regression algorithm $\mathcal{A}$.}
    \KwOut{\\ \quad Prediction interval $C(e)$ for each $e\in\mathcal{D}_{test}$.}
    Split links set into disjoint sets $\mathcal{D}_{train}, \mathcal{D}_{val}, \mathcal{D}_{calib}$, $\mathcal{D}_{test}$\;
    Construct subgraph $\mathcal{G}'=(\mathcal{V},\mathcal{E}')$, where $\mathcal{E}'=\mathcal{E}_p\cap(\mathcal{D}_{train}\cup\mathcal{D}_{val})$\;
    \tcp{Train the base model}
    \While{training}{
    Fit node embedding function: $f(\cdot) \leftarrow \mathcal{F(G')}$\;
    Fit edge embedding function: $z(\cdot) \leftarrow \mathcal{Z}(\{(f(u),f(v))|e=(u,v)\in\mathcal{D}_{train}\cup\mathcal{D}_{val}\})$\;
    Fit edge scoring function: $h(\cdot) \leftarrow \mathcal{H}(\{(z(e),y_e)|e\in\mathcal{D}_{train}\cup\mathcal{D}_{val}\})$\;}
    \tcp{Sampling}
    $\mathcal{D}'_{train}, \mathcal{D}'_{val}, \mathcal{D}'_{calib}\leftarrow Sampling(\mathcal{D}_{train}, \mathcal{D}_{val}, \mathcal{D}_{calib})$\;
    \tcp{CLP}
    Fit conditional quantile functions: $\{\widehat{\mu}_{\alpha/2}(\cdot), \widehat{\mu}_{1-\alpha/2}(\cdot)\} \leftarrow \mathcal{A}(\{(z(e),y_e)|e\in\mathcal{D}'_{train}\cup\mathcal{D}'_{val}\})$\;
    Compute the non-conformity score $V(e,y_e)=\max\{\widehat{\mu}_{\alpha/2}(z_e)-y_e, y_e-\widehat{\mu}_{1-\alpha/2}(z_e)\}$ for each $e\in\mathcal{D}'_{calib}$\;
    Compute $\widehat{q}$, the $\lceil (|\mathcal{D}'_{calib}|+1)(1-\alpha) /|\mathcal{D}'_{calib}|)\rceil$-$th$ empirical quantile of $\{V(e,y_e)|e\in\mathcal{D}'_{calib}\}$\;
    Construct prediction interval $C(e)=[\widehat{\mu}_{\alpha/2}(z_e)-\widehat{q}, \widehat{q}-\widehat{\mu}_{1-\alpha/2}(z_e)]$ for each $e\in\mathcal{D}_{test}$.
\end{algorithm}

\subsection{Sampling-based CQR for Improved Efficiency}
\label{sec:S-CQR}
Based on the findings above, we propose a simple yet effective sampling-based method for enhanced CP efficiency to quantify the uncertainty in GNN-based link prediction.

Our core idea is to bring the degree distribution of the existing graph into closer alignment with a power-law distribution, a modification that we believe will enhance the efficiency of CLP, as supported by our empirical research. One approach to achieve this is by selectively sampling specific edges such that the resulting node degree distribution closely follows the power-law distribution. We then use these sampled edges to compute nonconformity scores. Therefore, the first step involves obtaining an ideal degree sequence that adheres to a specific power-law distribution, serving as a reference. Subsequently, the sampling procedure is carried out, taking cues from this ideal degree sequence. The sampling process is detailed below.

\subsubsection{Fitting the power-law distribution}
Suppose that the node degree $d$ follows a discrete power-law distribution starting at $d_{min}\ge 1$, then the probability density function (PDF) of the power-law is defined as 
\begin{equation}
    \text{Pr}(d)=\frac{1}{\zeta(\beta, d_{min})}d^{-\beta}, 
\end{equation}
where $\zeta(\beta, d_{min})=\sum_{i=0}^{\infty}(i+d_{min})^{-\beta}$
is the Hurwitz zeta function and $\beta$ denotes the scaling exponent for power law distribution.
To determine the best-fitting power-law distribution for a given degree sequence, our primary objective centers on estimating $\beta$, which is the only unknown parameter in the PDF of the power-law distribution. Estimating $\beta$ requires selecting $d_{min}$, determined by the standard Kolmogorov-Smirnov minimization approach. This method identifies $d_{min}$ as the value minimizing the maximum absolute difference between the empirical distribution $E(d)$ and the cumulative distribution function of the best-fitting power law $P(d|\widehat{\beta})$ for degrees $d\ge d_{min}$~\cite{broido2019scale}.
The estimated $\widehat{\beta}$ is then obtained by \cite{clauset2009power}
\begin{equation}
    \widehat{\beta} = 1+n\left[ \sum_{i=1}^n \log\frac{d_i}{d_{min}-\frac{1}{2}}\right]^{-1}.
\end{equation}
%\lu{why do we need to estimate $\beta$? Explain here.}
\subsubsection{Generating ideal degree sequence with $\widehat{\beta}$-parameterized power-law distribution}
%In this step, we employ a power-law generating function such as the Pareto distribution to generate a degree sequence parameterized by $\widehat{\beta}$ and the number of nodes $n$. 
In this step, we generate a degree sequence adhering to power law distribution. 
Various distribution functions follow the power law. Here we utilize the Pareto distribution \cite{arnold2014pareto} as the specific power-law function to generate the degree sequence, parameterized by $\widehat{\beta}$ and the number of nodes $n$. 
Specifically, we define the PDF of the Pareto distribution in link prediction as follows
\begin{equation}
    f(x;x_m,\widehat{\beta}) = \frac{\widehat{\beta}\cdot x_m^{\widehat{\beta}}}{x^{\widehat{\beta}+1}},\quad \text{for } x\ge x_m,
\end{equation}
where $x_m$ is a scale parameter (minimum value for which the distribution is defined) and $\widehat{\beta}$ serves as a shape parameter (indicating the distribution's tail heaviness and skewness). 
%\lu{Why this function? what does it mean? How to generate the sequence?}

\subsubsection{Sampling edges from the original graph for a degree distribution that follows the power law}
We begin by computing the empirical cumulative distributions for both the degree sequence of the original graph and the ideal degree sequence.
Following this, we establish sampling probabilities for edges, determined by the deviation $dvia(d)$ between these two distributions.
Specifically, for a given edge $e=(u,v)$, we denote the degree of nodes $u$ and $v$ as $d_u$ and $d_v$, respectively.
The deviation $dvia(d)$ is calculated as:
\begin{equation}
    dvia(d) = \left |\mathrm{eCDF}_D(d)-\mathrm{eCDF}_{D'}(d)\right |,
\end{equation}
where $\mathrm{eCDF}_D(\cdot)$ and $\mathrm{eCDF}_{D'}(\cdot)$ represents the empirical CDFs of the original degree sequence $D$ and the ideal degree sequence $D'$, respectively.
Subsequently, the sampling probability of edge $e$ can be determined by:
\begin{equation}
    \mathbb{P}(e)=\min\{\lambda\cdot S(dvia(d_u),dvia(d_v)), 1\},
\end{equation}
where $\lambda>0$ is a hyperparameter and $S(\cdot)$ denotes the function of aggregating the two deviation scores, e.g., the operation of summation. 
In this process, we prioritize edges with greater deviations by assigning them a higher probability considering the deviation direction. To put this into action, we generate a random floating-point number, denoted as $r_e$, within the range of $[0,1)$ for each edge $e$. If $r_e\le\mathbb{P}(e)$, we retain this edge. Otherwise, we remove it from the original set of edges.
% \jian{what's the intuition of the equation above? need to explain here. e.g., it assigns higher probability to nodes with higher deviations? (2) after obtaining the sampling probabilities, what do you do? (sampling from existing edges?) how do you sample? (proportional to the probability?) these need to be explained clearly to readers.} 

The overall algorithm for CLP is presented in Algorithm~\ref{algo}. 
We first train a base GNN model for standard link prediction from line 3 to line 7.
Then, in line 8, the proposed sampling procedure is implemented.
Finally, we apply the conformalized link prediction method on the sampled edge set from line 9 to line 12 to obtain prediction intervals.
%\jian{need to describe the algorithm. e.g., in step xx, we do xxx; in steps xx -- xx, we do xx.}

%\begin{table*}[]
%\centering
%\caption{Statistics of the datasets.}
%\label{tab:dataset}
% \resizebox{\columnwidth}{!}{%
%\begin{tabular}{cccccc}
%\hline
%\textbf{Dataset Name} & ogbl-ddi  & ogbl-ppa   & ogbl-citation2 & german credit & rochester38 \\ \hline
%\textbf{\#Nodes}      & 4,267     & 576,289    & 2,927,963      & 1,000         & 4,563       \\
%\textbf{\#Edges}      & 1,334,889 & 30,326,273 & 30,561,187     & 22,242        & 167,653     \\
%\textbf{KS Statistic} & 0.3275    & 0.0908     & 0.0302         & 0.1133        & 0.1446      \\ \hline
%\end{tabular}%
% }
%\end{table*}

\begin{table}[]
\centering
\caption{Basic statistics of the datasets.}
\label{tab:dataset}
\begin{tabular}{c|ccc}
\hline
\textbf{Dataset Name} & \textbf{\#Nodes} & \textbf{\#Edges} & \textbf{KS Statistic} \\ \hline
ogbl-ddi              & 4,267            & 1,334,889        & 0.3275                \\
ogbl-ppa              & 576,289          & 30,326,273       & 0.0908                \\
ogbl-citation2        & 2,927,963        & 30,561,187       & 0.0302                \\
german credit         & 1,000            & 22,242           & 0.1133                \\
rochester38           & 4,563            & 167,653          & 0.1446                \\ \hline
\end{tabular}
\end{table}

\begin{table*}[]
\centering
\caption{Empirical coverage and average length of predictions intervals with target coverage 90\% (GCN backbone).}
\label{tab:result1}
\begin{tabular}{c|cc|c|c|c|c}
\hline
\textbf{dataset} &
  \multicolumn{2}{c|}{\textbf{KS Statistic}} &
  \textbf{method} &
  \textbf{emp. coverage (\%)} &
  \textbf{avg. prediction length} &
  \textbf{\begin{tabular}[c]{@{}c@{}}improved\\ efficiency\end{tabular}} \\ \hline
 &
  \cellcolor[HTML]{EFEFEF}before &
  \cellcolor[HTML]{EFEFEF}0.32 &
  \cellcolor[HTML]{EFEFEF}CQR for CLP &
  \cellcolor[HTML]{EFEFEF}91.79 $\pm$ 0.02 &
  \cellcolor[HTML]{EFEFEF}0.7656 $\pm$ 0.0135 &
   \\
\multirow{-2}{*}{ogbl-ddi} &
  after &
  0.24 &
  S-CQR for CLP &
  91.45 $\pm$ 0.05 &
  0.6321 $\pm$ 0.0252 &
  \multirow{-2}{*}{$\uparrow$ 17.43\%} \\ \hline
 &
  \cellcolor[HTML]{EFEFEF}before &
  \cellcolor[HTML]{EFEFEF}0.08 &
  \cellcolor[HTML]{EFEFEF}CQR for CLP &
  \cellcolor[HTML]{EFEFEF}90.31 $\pm$ 0.06 &
  \cellcolor[HTML]{EFEFEF}0.1720 $\pm$ 0.0400 &
   \\
\multirow{-2}{*}{ogbl-ppa} &
  after &
  0.04 &
  S-CQR for CLP &
  90.08 $\pm$ 0.01 &
  0.1659 $\pm$ 0.0010 &
  \multirow{-2}{*}{$\uparrow$ 3.54\%} \\ \hline
 &
  \cellcolor[HTML]{EFEFEF}before &
  \cellcolor[HTML]{EFEFEF}0.03 &
  \cellcolor[HTML]{EFEFEF}CQR for CLP &
  \cellcolor[HTML]{EFEFEF}90.09 $\pm$ 0.06 &
  \cellcolor[HTML]{EFEFEF}0.1428 $\pm$ 0.0013 &
   \\
\multirow{-2}{*}{ogbl-citation2} &
  after &
  0.02 &
  S-CQR for CLP &
  90.01 $\pm$ 0.12 &
  0.1264 $\pm$ 0.0322 &
  \multirow{-2}{*}{$\uparrow$ 11.48\%} \\ \hline
 &
  \cellcolor[HTML]{EFEFEF}before &
  \cellcolor[HTML]{EFEFEF}0.11 &
  \cellcolor[HTML]{EFEFEF}CQR for CLP &
  \cellcolor[HTML]{EFEFEF}91.43 $\pm$ 0.23 &
  \cellcolor[HTML]{EFEFEF}0.9552 $\pm$ 0.0200 &
   \\
\multirow{-2}{*}{german credit} &
  after &
  0.03 &
  S-CQR for CLP &
  91.49 $\pm$ 0.29 &
  0.7080 $\pm$ 0.0119 &
  \multirow{-2}{*}{$\uparrow$ 25.87\%} \\ \hline
 &
  \cellcolor[HTML]{EFEFEF}before &
  \cellcolor[HTML]{EFEFEF}0.14 &
  \cellcolor[HTML]{EFEFEF}CQR for CLP &
  \cellcolor[HTML]{EFEFEF}90.00 $\pm$ 0.02 &
  \cellcolor[HTML]{EFEFEF}0.8078 $\pm$ 0.0160 &
   \\
\multirow{-2}{*}{rochester38} &
  after &
  0.11 &
  S-CQR for CLP &
  90.14 $\pm$ 0.13 &
  0.4844 $\pm$ 0.0165 &
  \multirow{-2}{*}{$\uparrow$ 40.03\%} \\ \hline
\end{tabular}
\end{table*}

\section{Experiments}
In this section, we conduct experiments on real-world graph datasets across various domains (e.g., biology, citation network, and social network) to evaluate the performance of our proposed method. In particular, we aim to answer the following research questions:
\begin{itemize}
    \item Does the proposed CLP procedure attain the desired coverage in practical implementations?
    \item Does the proposed S-CQR for CLP effectively enhance CP efficiency?
    \item How does the proposed S-CQR perform across different GNN-based link prediction models?
    \item How does the involved hyperparameter impact the performance of CLP procedure?
\end{itemize}
\subsection{Experimental Setup}
\subsubsection{Datasets}
We evaluate the proposed CLP procedure on five benchmark datasets for link prediction, including the drug-drug interaction network \underline{ogbl-ddi}, protein interaction network \underline{ogbl-ppa}, citation network \underline{ogbl-citation2}~\cite{hu2020open}, and two social networks \underline{Ger-} \underline{man Credit}~\cite{agarwal2021towards} 
and \underline{Rochester38} \cite{traud2012social}.
This selection spans various graph scales, from smaller ones to large-scale graph datasets with millions of nodes, showcasing the wide-ranging applicability of our methods in real-world web contexts. 
The basic statistics of these datasets are shown in Table~\ref{tab:dataset}. We can see that the node degree distribution in ogbl-ddi dataset adheres least to the power-law distribution, followed by the Rochester38 dataset.

\subsubsection{Backbone Models}
We employ a three-layer Graph Convolutional Network (GCN)~\cite{kipf2017semi} and GraphSAGE~\cite{hamilton2017inductive} as the base models for link prediction. Note that any GNN models can be integrated into our proposed CP pipeline. CQR is implemented using neural networks for quantile regression, and the neural network architecture consists of three fully connected layers, with ReLU nonlinearities mapping between layers.

\subsubsection{Evaluation Setup}
For ogbl-ddi, ogbl-ppa, and ogbl-citation2 datasets, we use the splits given in the original papers~\cite{hu2020open}.
For German Credit and Rochester38, we split the links into sets as follows: 50\% for training ($\mathcal{D}_{train}$), 10\% for validation ($\mathcal{D}_{val}$), 20\% for calibration ($\mathcal{D}_{calib}$), and 20\% for testing ($\mathcal{D}_{test}$). 
We conduct five different random splits of calibration and test sets, and perform 10 repetitions of the experiment for each split. Averaged results are reported below.
We then measure empirical coverage and the average length of prediction interval to evaluate the validity and efficiency of both CQR for CLP and S-CQR (sampling-based) for CLP.
A detailed experimental setting is provided in Appendix~\ref{apdx:exp}.

It should be noted that CLP is a relatively recent research field. To the best of our knowledge, there exists only one prior study~\cite{marandon2023conformal} related to this issue. However, this work formalizes the problem in a different way from ours and thus cannot be directly used for comparison. Specifically, our work focuses on constructing prediction intervals that bound the miscoverage, while \cite{marandon2023conformal} focuses on bounding the false discovery rate.

\begin{table*}[]
\centering
\caption{Empirical coverage and average length of predictions intervals with target coverage 90\% (GraphSAGE backbone).}
\label{tab:result2}
\begin{tabular}{c|cc|c|c|c|c}
\hline
\textbf{dataset} &
  \multicolumn{2}{c|}{\textbf{KS Statistic}} &
  \textbf{methods} &
  \textbf{emp. coverage (\%)} &
  \textbf{avg. prediction length} &
  \textbf{\begin{tabular}[c]{@{}c@{}}improved\\ efficiency\end{tabular}} \\ \hline
 &
  \cellcolor[HTML]{EFEFEF}before &
  \cellcolor[HTML]{EFEFEF}0.32 &
  \cellcolor[HTML]{EFEFEF}CQR for CLP &
  \cellcolor[HTML]{EFEFEF}91.77 $\pm$ 0.03 &
  \cellcolor[HTML]{EFEFEF}0.7343 $\pm$ 0.0072 &
   \\
\multirow{-2}{*}{ogbl-ddi} &
  after &
  0.24 &
  S-CQR for CLP &
  91.83 $\pm$ 0.02 &
  0.6619 $\pm$ 0.0167 &
  \multirow{-2}{*}{$\uparrow$ 9.85\%} \\ \hline
 &
  \cellcolor[HTML]{EFEFEF}before &
  \cellcolor[HTML]{EFEFEF}0.08 &
  \cellcolor[HTML]{EFEFEF}CQR for CLP &
  \cellcolor[HTML]{EFEFEF}90.44 $\pm$ 0.03 &
  \cellcolor[HTML]{EFEFEF}0.1698 $\pm$ 0.0013 &
   \\
\multirow{-2}{*}{ogbl-ppa} &
  after &
  0.04 &
  S-CQR for CLP &
  90.10 $\pm$ 0.04 &
  0.1665 $\pm$ 0.0018 &
  \multirow{-2}{*}{$\uparrow$ 1.94\%} \\ \hline
 &
  \cellcolor[HTML]{EFEFEF}before &
  \cellcolor[HTML]{EFEFEF}0.03 &
  \cellcolor[HTML]{EFEFEF}CQR for CLP &
  \cellcolor[HTML]{EFEFEF}90.13 $\pm$ 0.11 &
  \cellcolor[HTML]{EFEFEF}0.1399 $\pm$ 0.0017 &
   \\
\multirow{-2}{*}{ogbl-citation2} &
  after &
  0.02 &
  S-CQR for CLP &
  90.02 $\pm$ 0.07 &
  0.1293 $\pm$ 0.0083 &
  \multirow{-2}{*}{$\uparrow$ 7.57\%} \\ \hline
 &
  \cellcolor[HTML]{EFEFEF}before &
  \cellcolor[HTML]{EFEFEF}0.11 &
  \cellcolor[HTML]{EFEFEF}CQR for CLP &
  \cellcolor[HTML]{EFEFEF}90.74 $\pm$ 0.28 &
  \cellcolor[HTML]{EFEFEF}0.9054 $\pm$ 0.0093 &
   \\
\multirow{-2}{*}{german credit} &
  after &
  0.03 &
  S-CQR for CLP &
  91.19 $\pm$ 0.15 &
  0.7402 $\pm$ 0.0125 &
  \multirow{-2}{*}{$\uparrow$ 18.25\%} \\ \hline
 &
  \cellcolor[HTML]{EFEFEF}before &
  \cellcolor[HTML]{EFEFEF}0.14 &
  \cellcolor[HTML]{EFEFEF}CQR for CLP &
  \cellcolor[HTML]{EFEFEF}90.03 $\pm$ 0.05 &
  \cellcolor[HTML]{EFEFEF}0.7065 $\pm$ 0.0142 &
   \\
\multirow{-2}{*}{rochester38} &
  after &
  0.11 &
  S-CQR for CLP &
  90.01 $\pm$ 0.01 &
  0.5099 $\pm$ 0.0236 &
  \multirow{-2}{*}{$\uparrow$ 27.82\%} \\ \hline
\end{tabular}
\end{table*}

\subsection{Main Results}

In Table~\ref{tab:result1}, we present the empirical coverage and average length of prediction intervals across five datasets with GCN as the backbone. 
% We further change the backbone model for link prediction to GraphSAGE to assess the performance of the proposed methods, as shown in Table~\ref{tab:result2}.
We can have the following observations according to the experimental results.
\subsubsection{The proposed Conformalized Link Prediction procedures achieve the target coverage.}
As shown in Table~\ref{tab:result1} and Table~\ref{tab:result2}, with the predefined error rate $\alpha=0.1$, both CQR and S-CQR for Conformalized Link Prediction achieve the desired coverage (90\%) on all five datasets. This empirically validates the theory established in Section~\ref{sec:exchang}. That is, in an inductive scenario, utilizing GNNs for link prediction on graphs adheres to a particular permutation invariance requirement and fulfills the exchangeability condition needed for CP. This enables the design of more advanced uncertainty quantification methods for link prediction integrated with CP techniques.
\subsubsection{The proposed sampling-based strategy effectively improves CP efficiency.}
Our proposed sampling strategy is very effective at improving the efficiency of the CLP process and assists in the generation of tighter prediction intervals while maintaining a desirable coverage rate. This result suggests that our proposed CLP approach can generate more informative prediction intervals in link prediction, which can be important in critical decision-making contexts. Particularly, we have the following observations: 
\textbf{Firstly}, we measure the KS statistic values for the graphs before and after the sampling operation outlined in Section~\ref{sec:S-CQR}. A lower KS statistic value suggests that the degree distribution of a graph aligns better with the characteristics of a power law distribution. The results are displayed in the `KS Statistic' column in Table~\ref{tab:result1}.
As we can see, the proposed sampling procedure can generate graphs with a node degree distribution that is more in line with the power law.
\textbf{Secondly}, the proposed S-CQR for CLP effectively reduces the average length of prediction intervals and increases the efficiency. This validates our hypothesis in Sec. 3.3 that with the same backbone GNN models, graphs that closely follow power law distribution typically lead to higher CP efficiency.
\textbf{Thirdly}, based on the results, it is evident that the proposed approach tends to perform more effectively on graphs whose degree distributions do not follow power law very well. For example, when assessing its performance on the ogbl-ppa and ogbl-citation2 datasets, which prominently follow the power law distribution (i.e., smallest KS statistics among the five datasets), the improvements in CP efficiency over the CQR are relatively modest. That is, the enhancements on these two datasets are the most modest among the five datasets evaluated, amounting to 3.54\% and 11.48\%, respectively. For the remaining datasets that do not closely adhere to the power law distribution, the observed enhancements in efficiency appear to be much more substantial.

%\subsubsection{Impact of backbone model selection on CLP performance.}
\subsubsection{CLP performance with different backbone link prediction models.}
%The performance of CP highly relies on the backbone machine learning models \cite{angelopoulos2021gentle,vovk2005algorithmic}: Models with more accurate predictions typically produce prediction sets/intervals smaller/shorter than less accurate models. 
%To examine the impact of backbone GNN models on the performance of the proposed CLP framework, we replace the GCN model with GraphSAGE and repeat the above experiments. Results in Table \ref{tab:result2} show that different GNN models can indeed have a great impact on the CLP performance.
%For instance, on ogbl-ddi dataset, though using both base models achieves target coverage, employing GCN as the base model leads to higher CP efficiency compared to GraphSAGE (with average prediction lengths of 0.7656 versus 0.7343), as detailed in Table~\ref{tab:result1} and Table~\ref{tab:result2}. This further impacts the performance of the proposed S-CQR algorithm, where employing GCN as the backbone resulted in an average prediction length of 0.6321, while using GraphSAGE as the backbone yielded an average prediction length of 0.6619.
To examine the impact of backbone GNN models on the performance of the S-CQR for CLP process, we replace the GCN model with GraphSAGE and repeat the above experiments. The results are shown in Table \ref{tab:result2}.
Our observations are as follows: The proposed approach consistently attains the target coverage rate and improves efficiency, i.e., reducing the length of prediction interval, across different backbone GNN models. For instance, on ogbl-ddi dataset, S-CQR decreases the prediction interval length from 0.7656 to 0.6321 with GCN and from 0.7343 to 0.6619 with GraphSAGE as the backbone models, meanwhile achieving the targeted 90\% coverage rate. 
The observed efficiency gains across different backbone models indicate that our method consistently boosts CP performance, demonstrating its backbone-model-agnostic effectiveness and robustness. This consistency aligns with expectations, as the link prediction process acts as a black box to the subsequent CLP procedure, suggesting the method's adaptability.

\begin{table}[]
\centering
\caption{Statistics of sampled graphs under different values of $\lambda$ on Rochester38 dataset.}
\label{tab:sample}
\begin{tabular}{c|cc}
\hline
$\lambda$ & graph density & KS Statistic \\ \hline
0.45     & 0.0143 & 0.1423       \\
0.40     & 0.0127 & 0.1374       \\
0.35     & 0.0112 & 0.1246       \\
0.30     & 0.0097 & 0.1167       \\
0.25     & 0.0082 & 0.0915       \\
0.20     & 0.0066 & 0.0719       \\
0.15     & 0.0049 & 0.0648       \\ \hline
\end{tabular}
\vspace{-2mm}
\end{table}

\begin{figure}[]
  \centering
  \includegraphics[width=0.8\linewidth]{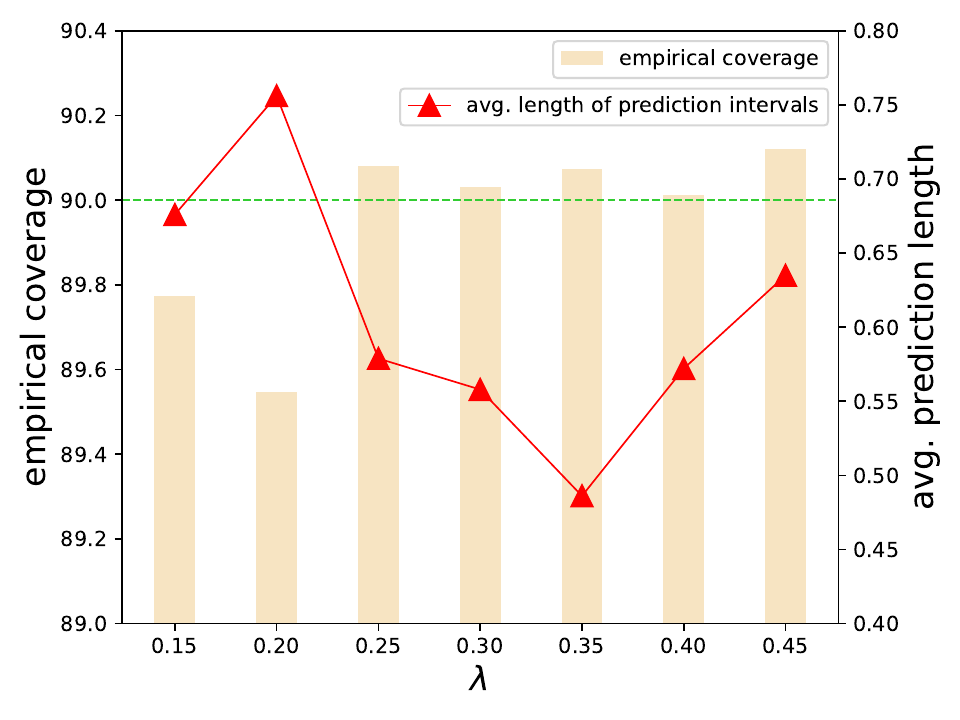}
  \vspace{-15pt}
  \caption{Performance of S-CQR for conformalized link prediction under different $\lambda$ on Rochester38 dataset.}
  %\Description{ablation german.}
  \label{fig:sample}
\end{figure}

\subsection{Comparison to Bayesian-based Uncertainty Quantification Methods}

\begin{table}[]
\centering
\caption{Performance of bayesian-based UQ methods.}
\label{tab:comp}
\resizebox{\columnwidth}{!}{%
\begin{tabular}{clcc}
\hline
\textbf{dataset} &
  \multicolumn{1}{c}{\textbf{method}} &
  \textbf{\begin{tabular}[c]{@{}c@{}}empirical\\ coverage\end{tabular}} &
  \textbf{\begin{tabular}[c]{@{}c@{}}average\\ prediction\\ length\end{tabular}} \\ \hline
\multirow{2}{*}{ogbl-ddi}      & MC Dropout & 0.9905 & 1.9999 \\ \cline{2-4} 
                               & BayesianNN & 0.9240 & 1.9865 \\ \hline
\multirow{2}{*}{german credit} & MC Dropout & 0.8921 & 1.9989 \\ \cline{2-4} 
                               & BayesianNN & 0.8978 & 1.8560 \\ \hline
\multirow{2}{*}{rochester38}   & MC Dropout & 0.9000 & 1.8023 \\ \cline{2-4} 
                               & BayesianNN & 0.9000 & 1.7956 \\ \hline
\end{tabular}%
}
\end{table}

To establish more baselines for comparisons, we further implement two of the most common bayesian-based uncertainty quantification methods: BayesianNN~\cite{gal2016dropout} and Monte Carlo Dropouts~\cite{kendall2017uncertainties}. The results on three datasets are presented in Table~\ref{tab:comp}.

Comparing Table~\ref{tab:comp} with Table~\ref{tab:result1} and Table~\ref{tab:result2}, we can observe that though both two Bayesian-based approaches can achieve or almost achieve the desired coverage rate, they yield much wider prediction intervals compared to our proposed method. This raises concerns about the efficiency of these Bayesian-based UQ methods. Additionally, their computational complexity is another significant concern.

\subsection{Analysis of Parameter $\lambda$}
%\lu{Add one more dataset for this experiment} 
To understand the effect of the hyperparameter $\lambda$ involved in the sampling process of S-CQR on its performance, we further conduct experiments on the Rochester38 dataset (randomly selected) applying different values of $\lambda$.
Adjusting the value of $\lambda$ impacts the density of the sampled graph. Specifically, a higher $\lambda$ yields a denser graph.
Specifically, we vary $\lambda$ among $\{0.45,0.4,0.35,0.3,0.25,0.2,0.15\}$ and perform the S-CQR for the CLP procedure.
The statistics of the resulting sampled graphs are presented in Table~\ref{tab:sample}.
As we can see, by controlling the value of $\lambda$, we can easily obtain edge sets exhibiting varying levels of adherence to the power law. A lower $\lambda$ leads to increased adherence and reduced edge density.

We further measure the performance of S-CQR for CLP, including empirical coverage and the average length of prediction intervals, under varying $\lambda$. The results are shown in Figure~\ref{fig:sample}.
We can observe that while an extremely small value of $\lambda$ tends to yield graphs with degree distributions aligning more closely with the power law, the inevitable decrease in edge density leads to a great loss of structural information. This subsequently results in a degraded CP performance, often manifested as an inability to attain the desired coverage and larger prediction intervals. These observations guide the selection of the optimal value for the parameter $\lambda$ during implementation.

\section{Related Works}
\subsection{Uncertainty Quantification on Graphs}
Graph-based machine learning models, especially in high-stakes scenarios, demand robust uncertainty quantification to avoid potentially costly errors. However, many current GNNs lack reliable uncertainty quantification methods, limiting their practical application. In previous studies, a common approach was adopting Bayesian techniques~\cite{goan2020bayesian,kingma2013auto,lakshminarayanan2017simple}. These methods aimed to obtain a distribution over network weights and quantify uncertainty through the posterior distribution. In the graph context, UAG~\cite{feng2021uag} used Bayesian uncertainty techniques to devise an uncertainty-aware attention mechanism to defend against adversarial attacks on GNNs. B-GCN~\cite{zhang2019bayesian} provided a way to integrate uncertain graph information using a parametric random graph model. GDC~\cite{hasanzadeh2020bayesian} tackled issues like over-smoothing and over-fitting commonly seen in deep GNNs, allowing for learning with uncertainty in graph analysis tasks and ultimately improving downstream task performance. However, Bayesian approaches, while theoretically sound, often encounter computational challenges. Additionally, the approximation methods for derivatives come with practical implementation drawbacks.

In recent years, conformal prediction~\cite{vovk2005algorithmic} has gained notable attention as a simple yet potent approach for producing statistically reliable uncertainty estimates. Nevertheless, conformal prediction has seen limited application in the context of graph-structured data, and the majority of existing studies have primarily focused on tasks related to node-level classification \cite{huang2023uncertainty,clarkson2023distribution,zargarbashi2023conformal} and regression~\cite{huang2023uncertainty}.

\subsection{Conformal Prediction on Graphs %\lu{Since you already talked about UQ on graphs, this section should review conformal prediction in general, e.g., in NLP, graphs, images}
}
Research efforts for applying conformal prediction to graph data have been relatively less. For instance, \cite{clarkson2023distribution} modifies existing conformal classification methods by incorporating network structure to adjust the conformal scores and introduces NAPS, a technique for constructing prediction sets for node classification in an inductive learning setting. Additionally, \cite{pmlr-v202-h-zargarbashi23a} introduces a conformal approach that provides prediction sets with distribution-free guarantees, making use of node-wise homophily in a transductive context. This approach updates conformal scores for each node based on neighborhood diffusion. Furthermore, \cite{huang2023uncertainty} investigates the exchangeability of node information in the transductive setting and introduces a permutation invariance condition that allows the conformal prediction to operate effectively on graph data. They also devise a topology-aware output correction model, CF-GNN, to enhance the efficiency of the conformal prediction procedure.

However, despite the progress in developing conformal prediction methods for node classification and regression, the application of such approaches to link-level tasks on graphs has remained under-explored.
In our research, we propose conformalized link prediction to further extend the conformal prediction procedure to link-level tasks on graphs and demonstrate its validity for link prediction under an inductive setting. Informed by the empirical analyses of synthetic data, we then propose a simple yet effective sampling-based method that leverages the structural properties of graphs to improve the efficiency of the standard conformal prediction pipeline. The key idea is to sample from the original graph to generate a new graph whose degree distribution aligns well with the power law distribution before applying the standard CQR. 
To the best of our knowledge, only one previous study has applied conformal prediction to link prediction on graphs~\cite{marandon2023conformal}. This study, however, conceptualizes the problem differently from our approach, with an emphasis on bounding the false discovery rate in contrast to our focus on bounding the miscoverage rate. This distinction precludes a direct comparison between the two studies.
\section{Conclusion}
In this research, we delve into the newly identified challenge of conformalized link prediction (CLP), which applies the principles of conformal prediction to graph neural network (GNN)-based link prediction tasks. Our focus is on validating the exchangeability assumption critical to CLP, for which we introduce a permutation invariance criterion tailored for link prediction, guaranteeing precise coverage at test time. Utilizing this criterion, we evaluate the feasibility of incorporating a standard conformal prediction framework, such as CQR, into CLP. We note a significant drawback in the direct use of CQR, namely its inefficiency, and uncover a crucial relationship between the graph's adherence to a power law distribution and the efficiency of CQR (i.e., the length of the prediction interval). This insight prompts us to develop a novel sampling-based conformal prediction technique that modifies the graph structure to align with a power law distribution, markedly enhancing the efficiency of conformal prediction. Our experimental findings reveal that this innovative method not only meets the desired coverage levels but also significantly narrows the prediction intervals when compared to existing approaches. Looking ahead, there is potential for creating more sophisticated conformal prediction strategies for CLP and expanding this framework to tackle node-level conformal prediction challenges.

%\section*{Acknowledgements}
\begin{acks}
This material is based upon work supported by the National Science
Foundation (NSF) Grant \#2312862 and a Cisco gift grant.  
\end{acks}

%\lu{add future works. such as extension to node classification}
%Experimental validation on real-world datasets demonstrates its effectiveness.

%\begin{acks}
%To Robert, for the bagels and explaining CMYK and color spaces.
%\end{acks}

\bibliographystyle{ACM-Reference-Format}
\balance
\bibliography{ref}

\appendix
\section{reproducibility}
\label{apdx:exp}

\begin{table}[h!]
\caption{Experimental settings for link prediction models.}
\label{tab:modellp}
\centering
\resizebox{\columnwidth}{!}{%
\begin{tabular}{c|ccccc}
\hline
\textbf{dataset} & \textbf{ogbl-ddi} & \textbf{ogbl-ppa} & \textbf{ogbl-citation2} & \textbf{german credit} & \textbf{rochester38} \\ \hline
training epochs  & 200            & 300            & 50   & 500  & 500  \\
learning rate    & 5e-3           & 1e-2           & 5e-4 & 1e-2 & 1e-2 \\
batch size       & 64$\times$1024 & 64$\times$1024 & 512  & 2048 & 2048 \\
hidden dimension & 128            & 256            & 256  & 128  & 128  \\ \hline
\end{tabular}%
}
\end{table}

\begin{table}[h!]
\centering
\caption{Experimental settings for quantile regression.}
\label{tab:qf}
\resizebox{\columnwidth}{!}{%
\begin{tabular}{c|ccccc}
\hline
\textbf{dataset} & \textbf{ogbl-ddi} & \textbf{ogbl-ppa} & \textbf{ogbl-citation2} & \textbf{german credit} & \textbf{rochester38} \\ \hline
training epochs  & 1000 & 200   & 50   & 200  & 200  \\
learning rate    & 5e-4 & 5e-4 & 5e-4 & 5e-4 & 5e-4 \\
batch size       & 64   & 64   & 64   & 64   & 64   \\
hidden dimension & 64   & 64   & 64   & 64   & 64   \\ \hline
\end{tabular}%
}
\end{table}

The experimental settings for the training of link prediction models and quantile regression are provided in Table~\ref{tab:modellp} and Table~\ref{tab:qf}, respectively.

\end{document}